\documentclass[letterpaper]{article} 
\usepackage[preprint]{aaai2027}
\usepackage[hyphens]{url}  
\usepackage{graphicx} 
\usepackage{natbib}  
\usepackage{caption} 
\usepackage{algorithm}
\usepackage{algorithmic}
\usepackage{booktabs}
\usepackage{amsmath,amssymb,amsthm}
\usepackage{acronym}
\acrodef{GNN}{graph neural network}
\acrodef{MPR}{multipoint relay}
\acrodef{OLSRv2}{Optimized Link State Routing Protocol version~2}
\acrodef{ILP}{integer linear program}
\acrodef{ETX}{expected transmission count}
\acrodef{MLP}{multilayer perceptron}
\acrodef{GATv2}{graph attention network v2}
\acrodef{OONF}{OLSR.org Network Framework}
\acrodef{PyG}{PyTorch Geometric}
\acrodef{NATO}{North Atlantic Treaty Organization}

\newtheorem{theorem}{Theorem}
\newtheorem{proposition}{Proposition}

\newtheorem{remark}{Remark}

\usepackage{xcolor}

\title{Learning to Cover Locally: Graph Neural Combinatorial Optimization\\
under a Hard Information Horizon}

\author{
    Johannes F. Loevenich\textsuperscript{\rm 1,\rm 2}\corresponding,
    Thies M\"ohlenhof\textsuperscript{\rm 1},
    Laurin Holz\textsuperscript{\rm 1},
    Maxime Schwarzer\textsuperscript{\rm 1},\\
    Tobias H\"urten\textsuperscript{\rm 1},
    Roberto Rigolin F. Lopes\textsuperscript{\rm 1}
}
\affiliations{
    \textsuperscript{\rm 1}cortAIx Labs, Thales Deutschland, Ditzingen, Germany\\
    \textsuperscript{\rm 2}Department of Mathematics/Computer Science, University of Osnabr\"uck, Osnabr\"uck, Germany\\
    johannes.loevenich@thalesgroup.com
}

\begin{document}
\maketitle

\begin{abstract}
Neural combinatorial optimization typically assumes a centralized solver that reads the whole instance. We study the opposite: combinatorial optimization under a hard information horizon, where every node commits to its share of a global solution seeing only its $k$-hop neighborhood, and those commitments must compose into a globally feasible solution. We formalize this as local set cover and instantiate it on weighted \ac{MPR} selection, the NP-hard 2-hop covering problem of the \ac{OLSRv2} routing protocol (RFC~7181), whose horizon is imposed by the protocol, not chosen by the modeler. We prove two results. Any deterministic selector whose horizon is one hop short must either fail coverage or land a factor $\Delta$ from optimal, and an $L$-layer \ac{GNN} read out at the deciding node is exactly an $L$-hop selector, so capacity cannot buy back radius. Conversely, at the horizon a \ac{GNN} of depth $O(\Delta)$ reproduces the RFC~7181 covering greedy, and at width $O(c_{\max}\Delta)$ its metric-aware weighted analogue, inheriting the $(1+\ln\Delta_2)$-approximation in both cases. Empirically, a 3-layer \ac{GATv2} with a coverage-completing decoder, behavior-cloned from the CP-SAT optimum, reaches $\text{cost}/\text{opt}=1.030\pm0.001$ against greedy's $1.138$, closing $79.1\%$ of the gap at $100\%$ coverage. Restricting the same learner to one hop, on identical instances with the same decoder and demonstrations, collapses it to $1.344$, far worse than greedy. Two transfer checks target real-world networks. OLSRv2's unmodified selection code matches our cardinality greedy on $200/200$ unit-cost instances, and on $40{,}308$ instances of real battalion mobility the frozen model closes $48\%$ of the gap at full coverage. The information horizon, not the model capacity, is the most significant variable.
\end{abstract}
\acresetall

\section{Introduction}
Most neural combinatorial optimization assumes a solver that sees the whole graph: a centralized model reads every node, edge, and cost, then emits one global decision \citep{khalil2017learning,chen2024learn}. We suppose the target is a set cover, but each node commits to part of it seeing only its $k$-hop neighborhood, and those commitments must compose into a valid global cover. We call this combinatorial optimization under a \emph{hard information horizon}, and take the horizon to be the object of study.

In a link-state routing protocol, each node selects a subset of its one-hop neighbors as relays so that every two-hop neighbor is covered, knowing nothing beyond two hops \citep{clausen2014olsrv2,maccari2018where}: the horizon is $k=2$. Choosing a minimum such set is NP-hard, and the choice is irrevocably local. The protocol, as a real network application, motivates the problem and a cost-efficient solution.

The horizon is sharp on both sides. Any deterministic policy restricted to a one-hop view is, on some instance, either infeasible or a factor $\Delta$ from optimal, for $\Delta$ being the maximum degree. Exactly at two hops, a bounded-depth \ac{GNN} reproduces the RFC~7181 greedy rule and inherits its $(1+\ln\Delta_2)$-approximation \citep{chvatal1979greedy,sato2019approximation}. The two-hop ball is what the protocol supplies, so the radius stays capped however many layers the network has. This means network depth buys computation, not radius. This ceiling represents the main limitation. For the cardinality objective, the greedy algorithm is already near-optimal. Therefore, the true potential for improvement lies in the weighted objective, where relays incur link costs (\ac{ETX}, energy) and even a reasonably cost-effective greedy approach leaves a substantial performance gap.

We make four key contributions. First, we formalize combinatorial optimization under a hard information horizon as local set cover on weighted relay selection \citep{rolnick2024application}. Second, we tightly bound the task via an \(\Omega(\Delta)\) floor below the horizon (Proposition~\ref{prop:impossible}) and constructive guarantees that a bounded-depth network reaches the greedy ceiling at the horizon for both objectives (Theorems~\ref{thm:constructive},~\ref{thm:wconstructive}). Third, we train a \ac{GATv2} selector \citep{brody2022how} using an \ac{ILP}-cloned, feasible-by-construction decoder, outperforming fair greedy where a one-hop restricted baseline fails. Fourth, the selector seamlessly transfers out of distribution and to real-world mobility at full coverage.

\section{Local Set Cover under a Hard Information Horizon}
Combinatorial optimization usually hands the solver the whole instance. We study the opposite, in which the instance is a graph and the solver is the graph itself: every node commits to part of a global solution while seeing only its $k$-hop neighborhood, and the local commitments must compose into one globally valid solution. We fix the horizon at $k=2$ and call the resulting task \emph{local set cover} (Figure~\ref{fig:overview}).

Fix one node $v$. Its $2$-hop ball defines a finite instance $(N_1, N_2, E, c)$. The candidate set $N_1$ collects the $1$-hop neighbors of $v$; its size is at most the maximum degree $\Delta$ (up to $34$ in our test instances). The terminal set $N_2$ collects the strict $2$-hop neighbors (up to $54$). We write $\Delta_2$ for the largest number of terminals covered by any single candidate (up to $24$ here), the quantity that governs the greedy $(1+\ln\Delta_2)$ guarantee rather than the terminal count $|N_2|$. A bipartite coverage relation $E\subseteq N_1\times N_2$ records which candidate reaches which terminal, and a link cost $c:N_1\to\mathbb{Z}_{>0}$ weights each candidate, where $c(u)$ is the integer metric of the link $\{v,u\}$: a selected relay is charged its own center-facing link cost, correlated with that link's geometric length. Node $v$ selects
\[
M^\star \in \arg\min_{M\subseteq N_1}\ \sum_{u\in M} c(u)
\;\text{s.t.}\; \bigcup_{u\in M}\{t:(u,t)\in E\} = N_2,
\]
the minimum-weight subset of candidates that covers every terminal. This is a weighted set cover problem on the local ball, and the horizon $k=2$ is a hard input constraint: information beyond the ball never enters the instance. Each node solves its own instance from its own view. The decisions are independent, since no center's selection constrains another's, and global feasibility is exactly the conjunction of the per-node coverage constraints, which is what makes an exact per-instance optimum available to supervise every center. The information horizon constrains the localized topological context of each node's instance rather than the coordination between centers; consequently, the union of the locally selected sets constitutes the global forwarding overlay.

The horizon places the task within the \textsc{local} model of distributed computation, where a node computes as a function of its $k$-hop view, and with the known correspondence between that model and graph neural networks \citep{sato2019approximation}. Each instance is a $2$-hop ball because the protocol supplies exactly that and nothing further: a node learns its $2$-hop neighborhood from neighbor-discovery exchanges and cannot acquire a third hop, so the extraction models the information a node has rather than a preprocessing convenience. Two consequences follow. The receptive radius is capped at $2$ hops however many message-passing layers the network stacks: depth buys computation over the fixed ball, not reach beyond it. And the horizon remains a quantity we impose on the \emph{learner}, not merely on the data: restricting the encoder's access to the ball's $2$-hop structure, with the instances, the decoder and the supervision all held fixed, lets us price the horizon empirically rather than only assert it. The weighted objective inherits the hardness of set cover, and the greedy rule \citep{chvatal1979greedy} supplies the reference $(1+\ln\Delta_2)$ guarantee against which local decisions are measured. Horizon and depth therefore decouple, exactly as Theorem~\ref{thm:constructive} requires, and together with Proposition~\ref{prop:impossible} this fixes a lower and an upper bound: below the horizon any deterministic policy is either infeasible or a factor $\Delta$ from optimal, while at the horizon a network of depth $O(\Delta)$ reproduces the greedy, its view never leaving the $2$-hop ball, and inherits the guarantee, for the cardinality objective (Theorem~\ref{thm:constructive}) and, at width $O(c_{\max}\Delta)$, for the weighted one (Theorem~\ref{thm:wconstructive}).

We instantiate local set cover on \ac{MPR} selection in the \ac{OLSRv2} \citep{clausen2014olsrv2}. Each router picks a subset of its symmetric one-hop neighbors ($N_1$) that together reach all of its two-hop neighbors ($N_2$); the chosen relays flood control traffic, and their union forms the routing backbone. Cardinality \ac{MPR} selection is the extensively studied covering problem \citep{maccari2018where}; assigning each relay a link cost $c$, a metric such as \ac{ETX} or energy, yields the weighted instance above, where each node's objective is the summed link cost of its own relay set rather than that set's size.

\begin{figure*}[t]
  \centering
  \includegraphics[width=\textwidth]{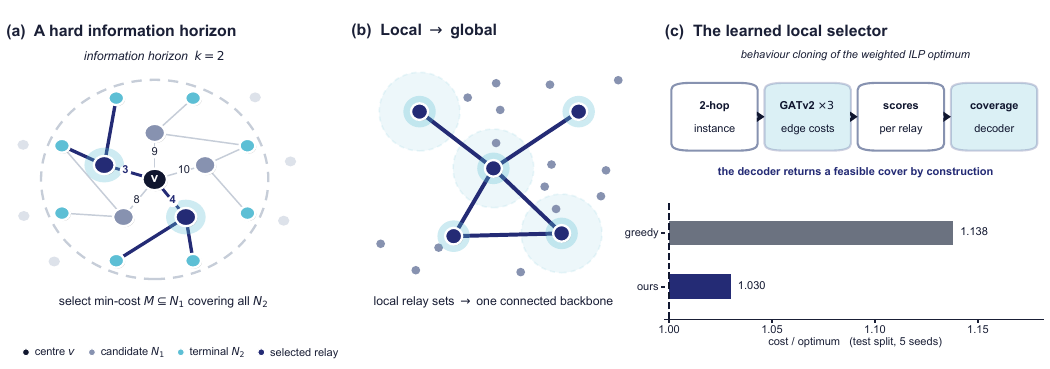}
  \caption{Combinatorial optimization under a hard information horizon.
  \textbf{(a)}~At centre $v$, candidates $N_1$ sit one hop out, each weighted by
  its centre-facing link cost, and terminals $N_2$ two
  hops out must all be covered; nothing beyond the dashed horizon exists for
  $v$. The costs make the objective weighted. \textbf{(b)}~Every node runs the
  same local selection and the relay sets union into one
  connected backbone; three local balls are drawn in full.
  \textbf{(c)}~An edge-conditioned GATv2 scores each candidate and a
  coverage-completing decoder returns a cover feasible by construction, so the
  network never learns feasibility. Below, cost/optimum against the fair greedy
  on the test split.}
  \label{fig:overview}
\end{figure*}

\section{What a $k$-Hop GNN Can Decide about a Global Cover}
\label{sec:theory}
Due to space constraints, full mathematical proofs and detailed network 
constructions for all propositions and theorems in this section are 
deferred to the supplementary material. Let $\Delta$ be the maximum degree and, for a center $v$, let $\Delta_2=\max_{x\in N_1}|\mathrm{cov}(x)|$ be the largest number of terminals a single candidate covers. Following the correspondence between message-passing \acp{GNN} and the distributed \textsc{Local} model~\citep{sato2019approximation}, a \ac{GNN} with $T$ layers computes, at each node, a function of its attributed $T$-hop ball.

\begin{proposition}[Horizon dichotomy]
\label{prop:impossible}
Call a selector \emph{one-hop-local} if its output at $v$ is a function of the
attributed ball $B_1(v)$. For every $\Delta\ge2$ there is a pair of instances of
maximum degree $\Delta$ on which any deterministic one-hop-local selector is
either infeasible, leaving a two-hop terminal uncovered, or feasible with
cardinality ratio at least $\Delta$. The construction uses unit costs, so the
dichotomy holds for the weighted objective too. Conversely, for the
cardinality objective selecting all of $N_1(v)$ is always feasible and, whenever
$N_2(v)\neq\emptyset$, has ratio at most $\Delta$; the worst-case ratio of
always-feasible one-hop selection over such instances is therefore exactly
$\Delta$.
\end{proposition}

\noindent\emph{Proof idea.} Two gadgets share an identical one-hop view of $v$ but
demand different relay sets: in the first, each candidate covers a private
terminal, forcing \emph{all} of them; in the second, all candidates share one
terminal, so \emph{any single} one suffices (Figure~\ref{fig:gadget}). A
deterministic one-hop-local selector returns the same set on both. If it omits
any candidate it misses that candidate's private terminal in the first gadget; if
it returns all $\Delta$, it pays $\Delta$ against an optimum of one in the
second. Tightness follows since $|N_1(v)|\le\Delta$. \hfill$\square$

\begin{figure}[t]
  \centering
  \includegraphics[width=0.9\linewidth]{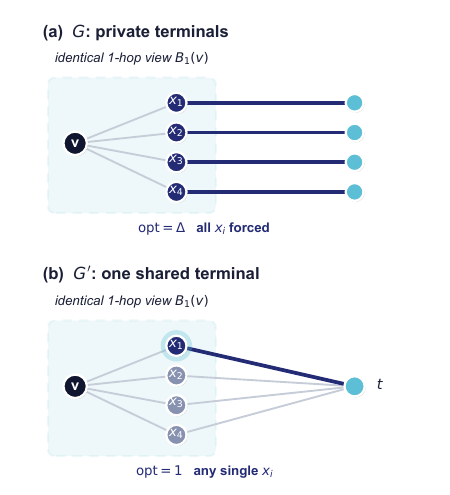}
  \caption{Both
  graphs give $v$ the identical one-hop view boxed in dashes.
  \textbf{(a)}~Each candidate covers a private terminal, forcing all $\Delta$
  into the relay set. \textbf{(b)}~All share one terminal, so any single
  candidate suffices. A selector restricted to one hop cannot distinguish them,
  so it is either infeasible in~(a) or a factor $\Delta$ from optimal in~(b).}
  \label{fig:gadget}
\end{figure}

\begin{remark}
An $L$-layer message-passing \ac{GNN} read out at $v$ computes a function of its
attributed $L$-hop ball \citep{sato2019approximation}, and is therefore an
$L$-hop-local selector. Proposition~\ref{prop:impossible} consequently applies to
every such one-layer network \emph{at any width}: capacity cannot substitute for
radius.
\end{remark}

\begin{theorem}[Constructive guarantee]
\label{thm:constructive}
For every $\Delta$ there is a message-passing \ac{GNN} with port numbering, depth
$O(\Delta)$, organized as at most $\Delta$ super-rounds of $O(1)$ layers of
which the $k^\ast\le\Delta$ greedy rounds are active, and $O(1)$ hidden
dimensions of $O(\log\Delta)$-bit precision, which at every center reproduces the
cardinality coverage greedy: marginal-gain selection under the fixed
deterministic key $(\text{gain},\text{willingness},\text{port})$ of our reference
implementation, the covering core of RFC~7181 multipoint-relay selection. It therefore attains a
$(1+\ln\Delta_2)$-approximation to the \emph{minimum-cardinality} MPR set, while
all message passing stays within the two-hop ball.
\end{theorem}

\noindent\emph{Proof.} Fix a center $v$ and let $U$ be the terminals still
uncovered. Three facts compose into the claim; each is a lemma of the supplement,
where the layer weights are exhibited. \emph{(i)~Counting.} One sum-aggregation
over the coverage edges gives every candidate its marginal gain
$g(x)=|\mathrm{cov}(x)\cap U|$ exactly, in $O(\log\Delta)$ bits.
\emph{(ii)~Selection.} Since $v$ is adjacent to every candidate, encoding the key
$(\text{gain},\text{willingness},\text{port})$ injectively into one integer, then
a max-aggregation at $v$ and a broadcast comparison, marks the unique key
maximizer and nothing else (the CPNGNN device of
\citealp{sato2019approximation}). \emph{(iii)~Update.} A gated layer adds the
marked candidate to $S$ if and only if its gain is positive, and removes
$\mathrm{cov}$ of it from $U$. Each is $O(1)$ layers, so a greedy round is $O(1)$
layers, and $\Delta$ rounds suffice because every active round covers at least
one terminal. Induction over rounds gives $S=\mathcal{G}(v)$, the reference
greedy's output, so Chv\'atal's bound transfers unchanged
\citep{chvatal1979greedy}. Every aggregation stays inside the two-hop ball:
depth buys rounds, not radius.\hfill$\square$

\begin{theorem}[Weighted constructive guarantee]
\label{thm:wconstructive}
For integer link costs in a fixed range $\{1,\dots,c_{\max}\}$ ($c_{\max}\!=\!10$
here), there is a message-passing \ac{GNN} with port numbering, depth $O(\Delta)$, width
$O(c_{\max}\Delta)$, and $O(\log\Delta)$-bit precision that reproduces the covering
loop of the weighted cost-effectiveness greedy and attains a
$(1+\ln\Delta_2)$-approximation to the minimum-\emph{weight} MPR set, within the
two-hop ball. Writing $g(x)$ for the number of still-uncovered terminals a candidate $x$ covers in
a given round, the ratio $c(x)/g(x)$ is a fraction of bounded integers, so encoding
it by its rank in a fixed finite set reduces the weighted selection to the integer
construction above, at width $O(1)\!\to\!O(c_{\max}\Delta)$.
\end{theorem}

\noindent\emph{Proof.} The weighted greedy selects $\arg\min_x c(x)/g(x)$. With
$c(x)\in\{1,\dots,c_{\max}\}$ and $g(x)\le\Delta$, that ratio takes values in a
fixed set $R$ of at most $c_{\max}\Delta$ rationals, so replacing the key by the
integer rank $\rho(x)=|R|-\mathrm{rk}(c(x)/g(x))$, with $\rho(x)=0$ when
$g(x)=0$, turns the argmin into an argmax over $O(\log(c_{\max}\Delta))$-bit
integers, which is exactly what step~(ii) above already selects. Steps~(i)
and~(iii) are untouched, so width $O(c_{\max}\Delta)$ suffices and the
$(1+\ln\Delta_2)$ bound follows from Chv\'atal's weighted analysis. The full
constructions, and the proof of Proposition~\ref{prop:impossible}, are in the
supplement.\hfill$\square$

\section{A Feasible-by-Construction Neural Selector}
A node at the horizon sees exactly its induced $2$-hop neighborhood. We encode that neighborhood as a graph and learn to read it. Each local \ac{MPR} instance becomes a \ac{PyG} graph whose vertices carry a role (center, candidate, or terminal), the node degree, the willingness parameter of \citet{clausen2014olsrv2}, a coverage-degree counting how many uncovered terminals a candidate reaches, a terminal-reach indicator, and two link-cost features. The charged cost is the center-facing link metric, which means a selected relay pays the cost of its own link to $v$. The costs of the links by which a candidate reaches its terminals are not charged; they enter only as auxiliary features. Both kinds of cost live on edges, so the encoder must read edge attributes, not node summaries alone, to consume them.

The encoder is a GATv2 attention network~\citep{brody2022how} of $3$ layers and roughly $240\text{k}$ parameters, chosen because GATv2 conditions each attention weight on the incident edge and therefore consumes the link costs directly rather than through a lossy node summary. Its output is one scalar score per candidate relay. A score is not a cover, and an uncalibrated score can miss a terminal entirely; we therefore never let the network decide feasibility. A coverage-completing decoder consumes the scores in a single pass: it computes the scores once, sorts the candidates once by descending score, force-selects any always-forward (\textsc{will\_always}) relays, and then, while an uncovered terminal remains, commits the first candidate in that fixed order that covers at least one still-uncovered terminal. Ties in score are broken deterministically by the decoder's fixed candidate enumeration. The loop halts only when every $2$-hop terminal is covered, so every decoded solution is a valid local cover \emph{by construction}. Feasibility is a property of the decoder, not a target the optimizer must hit; the network learns an ordering, and the decoder turns any ordering into a $100\%$-feasible cover.

This decoder is the constructive theorem made executable. The bounded-depth port-numbered network of that result reproduces the RFC~7181 coverage greedy of \citet{clausen2014olsrv2} and inherits its $(1+\ln \Delta_2)$-approximation \citep{chvatal1979greedy}; our decoder is the \emph{static-score variant} of that skeleton: where the theorem's construction recomputes an adaptive marginal-gain key each round, ours ranks candidates once by a fixed learned score and relies on the feasibility skip to preserve coverage-completion. In the classical cardinality regime the skeleton alone is near-optimal and there is little to learn, an empirical property of our distribution, where unweighted greedy sits within $0.5\%$ of optimal, not something the $(1+\ln\Delta_2)$ bound predicts. The learning earns its keep in the weighted regime, where relays carry link costs such as ETX or energy and Theorem~\ref{thm:wconstructive} certifies only the fixed cost-effectiveness rule. Chv\'atal's $(1+\ln\Delta_2)$ bound covers the weighted greedy no less than the cardinality one, so the baseline is not classically worse; what it leaves is an \emph{empirical} gap on realistic metric distributions, invisible to the worst-case bound, that the learned score closes. Two caveats keep this honest. First, replacing the fixed cost-effectiveness rule with a learned score forfeits the theorem's guarantee: the $(1+\ln\Delta_2)$ analysis is a property of the greedy \emph{rule}, so once the ordering is learned the decoded policy's only \emph{provable} property is feasibility, and its near-optimality on the weighted objective is an empirical finding, not a theorem. Second, the decoder externalizes the theorem's $k^\star$ sequential rounds: the coverage loop runs outside the network, so the encoder produces one ordering in a single forward pass rather than realizing the in-network round structure of Theorem~\ref{thm:constructive}, which is why a $3$-layer GATv2 suffices where that construction needs depth $O(\Delta)$. The depth the theorem spends on sequential composition, our design spends in the decoder's loop.

We train by behavior cloning. For each instance we solve the weighted MPR integer program to optimality with CP-SAT, the optimum subject to the protocol's always-forward constraint, so both the demonstration and the evaluation yardstick are the protocol-legal optimum, and treat the optimal relay set as the demonstration; the network is fit to rank the chosen relays above the rejected ones. Every candidate in our instances carries the default willingness, so that constraint is vacuous and the two optima coincide. We choose imitation over reinforcement learning deliberately. The instances are tiny, a single node's $2$-hop neighborhood, so an exact ILP expert is cheap to query at scale, which makes a dense supervised signal available for every training graph. Reinforcement learning would discard that signal, replace it with a sparse decoded-cost reward, and inherit the credit-assignment and stability problems of neural combinatorial optimization \citep{khalil2017learning,chen2024learn}; behavior cloning of an available optimum sidesteps all of it. Model selection uses the quantity we ultimately care about: we retain the checkpoint with the lowest decoded validation cost ratio, so validation measures the deployed decode, not a surrogate loss.

The architecture makes the division of labor explicit. The horizon fixes what can be seen; the theorem fixes that a greedy skeleton at that horizon is the right computation; the decoder supplies that skeleton and guarantees coverage; and the GATv2 encoder supplies the one thing the fixed rule lacks, a cost-aware ordering learned from optimal demonstrations.

\section{Experiments}
We evaluate the central hypothesis that within the 2-hop information horizon, a learned local policy can capture optimization gains that the most effective handcrafted heuristic fails to recover. Here, the local horizon serves as the primary subject of investigation rather than a system constraint. Under this paradigm, each node determines its localized contribution to the global cover based exclusively on its 2-hop neighborhood, a radius characterized by Theorems~\ref{thm:constructive} and~\ref{thm:wconstructive} as both a fundamental lower bound and a strict upper bound on performance.

\paragraph{Experimental Setup.} We evaluate on random-geometric unit-disk graphs, the standard model of a wireless neighborhood: $n$ nodes drop uniformly on the unit square, two link when within radius $r$, and each link carries an integer metric correlated with its length, mimicking ETX or energy costs. We optimize the \emph{weighted} \ac{MPR} objective, in which a relay set pays the sum of the selected relays' center-facing link metrics rather than its cardinality. This is where local reasoning pays off. For cardinality MPR, the greedy algorithm is near-optimal, and the ceiling is analytic; however, once relays are weighted by costs, even the fair cost-effectiveness greedy approach leaves a real gap. A CP-SAT integer program solves each instance to optimality and fixes the benchmark at $\text{cost}/\text{opt}=1.0$.

\paragraph{Primary Empirical Results.} Evaluated across 3,200 in-distribution test instances (\(n=80\), \(r=0.30\)) over five distinct random seeds, the proposed GNN achieves a competitive cost ratio of \(\text{cost}/\text{opt}=1.030\pm0.001\), compared to \(1.138\) for the greedy baseline. This performance translates to closing \(79.1\%\pm0.6\%\) of the greedy-to-optimal performance gap, with the GNN matching or outperforming the greedy heuristic on \(93.5\%\) of the instances. Notably, the variance across training seeds is negligible, remaining an order of magnitude below any reported effect size. The gap-closed metric is averaged exclusively over the subset of instances where the greedy baseline is strictly suboptimal (\(58.9\%\) of the test set), as the ratio remains undefined elsewhere. Verified via an independent oracle, every decoded solution successfully covers all 2-hop terminals; this continuous global feasibility is guaranteed by the architectural design of the decoder rather than being an emergent learned property. Furthermore, the mean performance ratio understates the model's efficacy: the GNN selector achieves exact optimality on \(85\%\) of all instances, in contrast with the greedy baseline's \(41\%\), and consistently matches or bounds the greedy performance across the entire empirical distribution (Figure~\ref{fig:analysis}a).

\begin{figure*}[t]
  \centering
  \includegraphics[width=\textwidth]{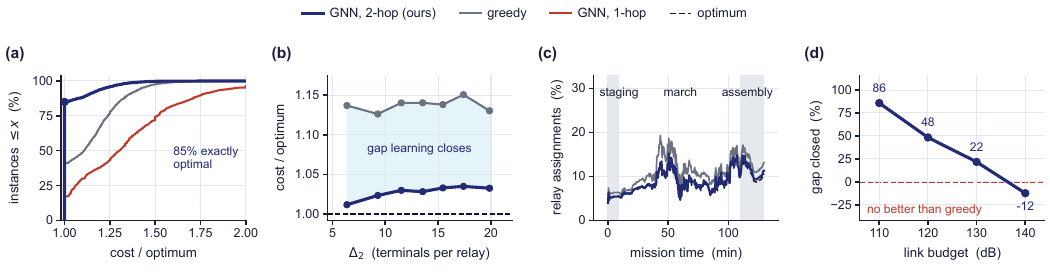}
  \caption{\textbf{(a)}~Per-instance
  cost/optimum on the $3200$ test instances.
  \textbf{(b)}~Cost/optimum against $\Delta_2$.
  \textbf{(c)}~Relay assignments summed over flood sources, as a percentage of
  full flooding's, over the $130$-minute ANGLOVA mission ($40{,}308$ instances
  sampled every $30$\,s); a per-source sum, not a transmitter count. \textbf{(d)}~Link-budget sweep. Panels (a,~b) use the released
  checkpoint; each point is a full-split aggregate.}
  \label{fig:analysis}
\end{figure*}

\paragraph{Impact of the Information Horizon.} The structural significance of the information horizon is verified by restricting the encoder to a 1-hop ball. We remove all coverage edges and 2-hop-dependent features, specifically zeroing out candidate degrees that reflect terminal coverage metrics, while keeping the feasibility-enforcing decoder intact. Following retraining, the center-facing link cost remains the unique candidate attribute. This 1-hop variant yields a cost ratio of \(\text{cost}/\text{opt}=1.344\pm0.000\), marking a substantial performance decline relative to the 2-hop model (\(1.030\)) and the greedy baseline (\(1.138\)) (Figure~\ref{fig:results}a). Deterministic convergence across all five seeds arises because the single scalar feature permits only two monotonic permutations, with every seed optimizing for descending link costs. This learned ordering fails to replicate the greedy heuristic; although cost correlates with coverage scope (Pearson \(r=0.56\) overall, \(0.61\) intra-instance), it acts only as an insufficient proxy for the latent graph topology. Thus, reducing the observation radius by one hop under identical architectural and supervisory conditions severely degrades policy efficacy. This ablated system circumvents the infeasibility branch of Proposition~\ref{prop:impossible} because the decoder continues to operate over the 2-hop ball to maintain \(100\%\) feasibility, isolating the empirical evaluation strictly to optimality bounds.

\paragraph{Component Ablation (E1).} We decompose the model's performance to isolate the sources of optimization gain (Table~\ref{tab:e1}). A random feasible policy establishes an upper bound cost ratio of \(1.641\). A structure-blind MLP scorer in the style of DeepMPR \citep{kaviani2023deepmpr} reduces this to \(1.218\), underperforming the greedy baseline (\(1.138\)) and demonstrating that data-driven optimization without topological awareness is insufficient. Incorporating message passing over the 2-hop subgraph using node-summary edge costs achieves \(1.057\), whereas edge-conditioned attention reduces the ratio further to \(1.029\). These results confirm that explicit edge-attributed graph modeling is critical. Analyzing the suboptimality gap as a function of \(\Delta _{2}\) reveals distinct scaling behaviors: the greedy heuristic's performance remains invariant at approximately \(1.14\), whereas the GNN cost ratio varies marginally between \(1.012\) and \(1.033\) as individual relays cover more terminals (Figure~\ref{fig:analysis}b). Every metric is normalized against the exact CP-SAT local optimum (\(1.000\)), defining a strict performance floor for any local selector. Finally, a full-graph baseline is omitted here as it evaluates a coupled global objective.

\begin{table}[t]
\centering
\begin{tabular}{lc}
\toprule
Method & $\text{cost}/\text{opt}$ \\
\midrule
Random feasible & $1.641$ \\
MLP, no graph (DeepMPR-style) & $1.218$ \\
Greedy reference & $1.138$ \\
GNN, node-summary costs & $1.057$ \\
GNN, full (ours) & $\mathbf{1.029}$ \\
Optimum (CP-SAT) & $1.000$ \\
\bottomrule
\end{tabular}
\caption{E1 ablation ladder on the test set ($n=3200$).}
\label{tab:e1}
\end{table}

\begin{figure}[t]
  \centering
  \includegraphics[width=\linewidth]{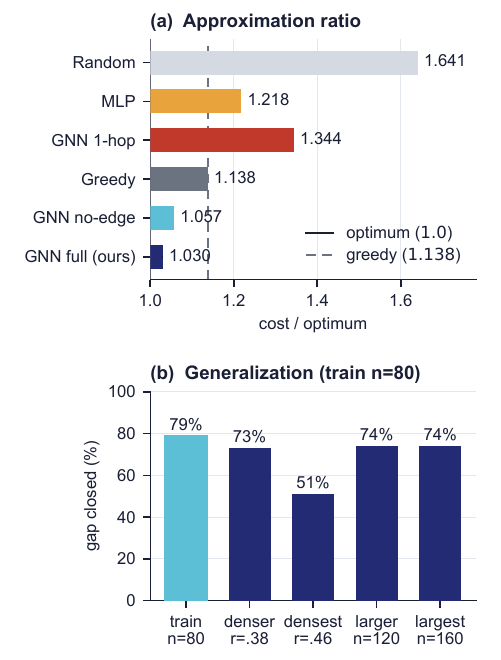}
  \caption{\textbf{(a)}~Cost/optimum ($1.0$ = optimal) for every method.
  \textbf{(b)}~Out-of-distribution transfer of the $n{=}80$ model.}
  \label{fig:results}
\end{figure}

\paragraph{Structural Feasibility Verification (E2).} Decoupling feasibility from optimization ensures a perfect coverage ratio of \(1.0\) for any score-generating heuristic wrapped by the decoder, including the random and MLP models, whereas raw neural outputs fail to satisfy coverage constraints independently. Because constraint satisfaction is an architectural invariant of the decoding mechanism rather than a learned objective, the neural encoder allocates its full capacity exclusively to cost optimization rather than maintaining global validity.

\paragraph{Generalization and Transferability (E3).} The architecture's generalization capacity is evaluated by applying the model trained at \(n=80, r=0.30\) to different topological domains. Scaling the node count to \(n=160\) maintains a high performance, with the gap-closed metric varying marginally from \(\sim79\%\) to \(\sim74\%\). Under dense network conditions (\(r=0.46\)), the metric decreases to \(51\%\), preserving a significant margin over the greedy heuristic and achieving a \(100\%\) feasibility rate across all evaluations (Figure~\ref{fig:results}b). This generalization capability stems from the local horizon design: the network captures localized 2-hop topological features that are statistically invariant to changes in global network scale.

\paragraph{Protocol Validation and Real-World Mobility (E4).} 
To validate our abstraction, we evaluate the 2-hop cardinality greedy implementation against a standalone harness driving the unmodified RFC~7181 selection code from \ac{OONF}~\citep{oonf}. Across 200 unit-cost instances, the selections match identically, confirming implementation fidelity (RFC~7181 defines no weighted variant). 

For real-world mobility, we utilize the public IST-124 ANGLOVA scenario~\citep{suri2018anglova}, a \ac{NATO} battalion vignette, extracting 40,308 non-trivial local instances across 260 timesteps sampled every 30\,s (120\,dB link budget). Evaluated on this dataset, the frozen synthetic-trained checkpoint achieves a cost ratio of 1.033 versus greedy's 1.080, closing 48.3\% of the performance gap at 100\% feasibility and outperforming greedy on 92.8\% of instances. Performance varies with node density, ranging from 80.9\% gap-closed during sparse staging to 7.2\% in dense assembly phases. Summed over every center as a flood source, the learned relay sets account for 8.7\% of full flooding's one-hop forwarding assignments, against 10.7\% for greedy and 8.3\% for the per-center optimum, and they track that optimum below greedy at every timestep (Figure~\ref{fig:analysis}c). This per-source sum is the quantity the per-node objective controls; it is not a count of transmitters, since a node that relays for several centers is counted once per center. A link-budget sweep demonstrates architectural robustness, yielding 85.8\% gap-closed at 110\,dB before dropping to $-12.4\%$ at the 140\,dB noise threshold where the greedy heuristic becomes competitive (Figure~\ref{fig:analysis}d).

\paragraph{Experimental Specification.} Local instances are mapped to \ac{PyG} graphs featuring node roles, vertex degrees, willingness, coverage degrees, terminal-reach indicators, and two edge-level link costs. The dataset comprises 12,000 training, 2,400 validation, and 3,200 test instances drawn from disjoint graph seeds. The encoder utilizes a 3-layer \ac{GATv2} with \(\approx \)240k parameters. Training employs behavior cloning derived from the weighted \ac{ILP} optimum via the CP-SAT solver, which establishes the reference baseline; model checkpoints are selected based on the lowest decoded validation cost ratio.In terms of computational latency, the median per-instance runtimes are 1.7 ms for CP-SAT, 0.06 ms for greedy, and 5.6 ms for our forward pass and decode loop. This work targets solution quality under localized information horizons rather than computational acceleration over small-scale exact solvers. Primary and ablation results denote the mean \(\pm \) standard deviation over five 12-epoch seeds, whereas the E1 ladder utilizes a single 12-epoch run, and E3, E4, and latency benchmarks rely on a 20-epoch checkpoint. The evaluation framework reproduces all aggregate metrics identically. The \ac{OONF} harness, ANGLOVA loader, and evaluation scripts are available in the public code repository.

\section{Related Work}
Learning to solve combinatorial optimization on graphs has, until now, typically meant learning on the \emph{whole} graph. S2V-DQN pairs a graph embedding with a $Q$-learner that grows a partial solution one node at a time while reading the global state at every step \citep{khalil2017learning}; a recent \ac{GNN} coupled with deep reinforcement learning learns minimum dominating sets over the entire topology \citep{chen2024learn}. Both assume a centralized solver with unobstructed access to all $n$ nodes. We remove that assumption. Each node commits to part of the cover while seeing only its $k=2$-hop neighborhood, and the local commitments must compose into a globally valid cover. The information horizon is the object of study, not an implementation detail: below two hops, any policy suffers an $\Omega(\Delta)$ gap, and at two hops a bounded-depth \ac{GNN} reproduces the RFC~7181 cardinality greedy.

The horizon also separates us from learned routing in mobile ad hoc networks. DeepMPR replaces the relay heuristic with a per-packet forwarding decision produced by a multilayer perceptron, discarding the multipoint-relay set entirely \citep{kaviani2023deepmpr}; graph-reinforcement approaches learn message-passing policies benchmarked against OLSR relaying \citep{galliera2024learning}. These methods learn \emph{when to forward}, not \emph{which relay set covers the two-hop neighborhood at minimum cost}. We retain the relay-set construct and formulate its selection as a combinatorial optimization problem. This enables a coverage-completing decoder to guarantee $100\%$ feasibility and allows our ablation to show that a graph-blind per-candidate scorer performs worse than the greedy method it is intended to replace.

Relay selection has a long optimization history. The RFC~7181 rule is greedy \citep{clausen2014olsrv2}, and the surrounding literature layers metaheuristics onto that skeleton \citep{maccari2018where}. Greedy is the right reference point: Chv\'atal's cost-effectiveness greedy attains a $(1+\ln \Delta_2)$-approximation for the \emph{weighted} set-cover objective, and hence for its unit-cost cardinality special case, so its worst-case guarantee does not distinguish the two regimes \citep{chvatal1979greedy}. Theorem~\ref{thm:wconstructive} shows that a bounded-depth \ac{GNN} attains that ceiling without ever leaving the $2$-hop ball. What separates the two objectives is empirical: on our distribution, unweighted greedy is within $0.5\%$ of optimal, whereas weighted greedy leaves a substantial gap that a metric-aware \ac{GNN} reduces, reaching $\text{cost}/\text{opt}=1.030$.

Two expressivity results ground the theory. The correspondence between message-passing \acp{GNN} and the \textsc{Local} model of distributed computing characterizes what a $k$-layer network can and cannot observe \citep{sato2019approximation}, and it supports our impossibility result: an $L$-layer network read out at the deciding node is an $L$-hop-local selector, so the horizon dichotomy applies directly to the one-layer case, regardless of width. Depth--width lower bounds \citep{loukas2020what} constrain what can be computed \emph{above} the horizon and are complementary here. This allows us to state the horizon as a two-sided theoretical bound rather than only an empirical observation.

Covering under a bounded horizon is itself classical in that model. Kuhn, Moscibroda, and Wattenhofer establish tight locality-versus-approximation trade-offs for covering and packing LPs \citep{kuhn2016local}, Suomela's survey maps the broader landscape of local algorithms \citep{suomela2013survey}, and the \ac{GNN} approximation ratios we invoke descend from this line of work \citep{sato2019approximation}. Our contribution is not to reopen that theory, but to carry its horizon constraint into \emph{neural} combinatorial optimization: a learned selector that operates strictly within the horizon, a \emph{weighted} objective on which classical local greedy leaves a meaningful empirical gap, and a cross-check of our cardinality greedy against the deployed protocol's own selection code on unit-cost instances (RFC~7181 specifies no weighted variant).

Finally, our optimality yardstick is an exact integer program, not another heuristic, so every reported gap is measured against ground truth rather than against a competing approximation.

\section{Conclusion}
This work examined decentralized combinatorial optimization under a hard information horizon, requiring independent local commitments within a \(2\)-hop neighborhood to form a globally valid cover. For weighted \ac{OLSRv2} \ac{MPR} selection, this task encompasses two distinct optimization tasks. In the cardinality setting, the horizon defines a strict performance boundary: below it, deterministic policies suffer from a \(\Delta \)-factor optimality gap or infeasibility (Proposition~\ref{prop:impossible}), whereas a bounded-depth \ac{GNN} exactly matches the RFC~7181 greedy performance and its \((1+\ln \Delta_2)\)-approximation (Theorem~\ref{thm:constructive}).
Although Theorem~\ref{thm:wconstructive} successfully carries this upper bound over to the weighted case, the greedy heuristic remains sub-optimal. Machine learning yields significant utility in this domain: a GATv2 architecture cloned from the weighted \ac{ILP} optimum utilizes a feasible-by-construction decoder to close \(79.1\%\) of the greedy performance gap at full coverage. Empirical evaluations on the ANGLOVA tactical mobility benchmark confirm that the model generalizes out-of-distribution while guarantees on feasibility hold continuously. Two key insights remain open for exploration: the global union of selected sets creates a coupled network structure that could benefit from an increased observation radius, and the remaining weighted optimization gap may be sensitive to specific architectural choices.

\bibliography{references}

\end{document}


\maketitle

This supplement gives the proofs deferred from the main paper: a complete proof of
the horizon dichotomy (Proposition~1 in the paper, restated here as
Proposition~\ref{prop:dich}) and the full lemma-by-lemma construction behind the
constructive guarantee (Theorem~1 in the paper, restated as
Theorem~\ref{thm:constr}), which the main paper establishes at proof-sketch level;
\S\ref{sec:weighted} then proves the weighted counterpart (Theorem~2 in the paper,
restated as Theorem~\ref{thm:wconstr}). Notation follows the main paper.
Fix a node $v$ of a connected graph $G$ of maximum degree $\Delta$. Its two-hop
ball supplies the local instance: $N_1=N_1(v)$ are the one-hop neighbours
(\emph{candidates}), $N_2=N_2(v)$ are the strict two-hop nodes (\emph{terminals},
i.e.\ nodes at distance exactly $2$ from $v$), and for a candidate $x$ we write
$\cov(x)=\{t\in N_2:\{x,t\}\in E\}$. An \emph{MPR set} is $M\subseteq N_1$ with
$\bigcup_{x\in M}\cov(x)=N_2$; $\opt(v)$ is the minimum cardinality of such a set.
Write $\Delta_2=\max_{x\in N_1}|\cov(x)|$. We call an instance \emph{non-trivial}
if $N_2\neq\emptyset$ (so $\opt(v)\ge 1$). In the \emph{weighted} setting
(\S\ref{sec:weighted}) each candidate additionally carries an integer cost
$c(x)\in\{1,\dots,c_{\max}\}$, and $\opt_w(v)$ denotes the minimum total cost of
a cover.

\section{An instance where the fair greedy loses}

\begin{figure}[h]
  \centering
  \includegraphics[width=\linewidth]{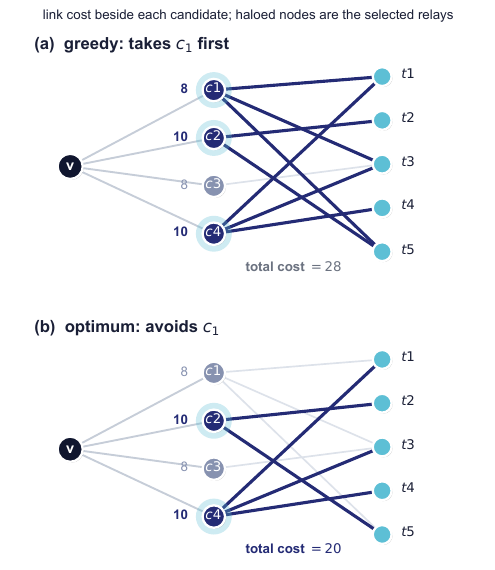}
  \caption{The fair cost-effectiveness
  greedy grabs $c_1$ first, it having the best cost-per-coverage, and finishes at
  cost $28$. The optimum avoids $c_1$ entirely: $\{c_2,c_4\}$ covers the same
  five terminals for cost $20$, a $29\%$ saving. The learned selector recovers
  exactly this kind of cost-aware choice, which no locally-greedy rule can see.
  The instance is node $16$ of seed $130$ drawn from the same generator as the
  test split at $n=45$, $r=0.30$ (the released test split uses $n=80$); it is an
  illustration, not a sampled statistic.}
  \label{fig:vignette}
\end{figure}

The greedy's first pick is locally rational and globally wrong: $c_1$ has the
best ratio of cost to fresh coverage, but taking it strands $t_2$ and $t_4$
behind two more relays. This is the weighted regime's characteristic failure,
and it is invisible to the cardinality objective, where $c_1$ is simply a good
pick.

\section{Proof of the horizon dichotomy}
\label{sec:prop}

\begin{definition}[one-hop-local selector]
A deterministic selector $A$ is \emph{one-hop-local} if its output $A(v)\subseteq
N_1(v)$ is a function only of the attributed one-hop ball $B_1(v)$: the subgraph
induced by $\{v\}\cup N_1(v)$ together with its node and edge features and, when
the model is port-numbered, the ports of the edges incident to $v$. In
particular $A(v)$ does not depend on $N_2(v)$ or on the coverage sets
$\cov(\cdot)$, which are determined only by $B_2(v)$.
\end{definition}

\begin{proposition}[horizon dichotomy]
\label{prop:dich}
For every integer $\Delta\ge 2$ there is a pair of non-trivial instances, each of
maximum degree $\Delta$, that present $v$ with an identical one-hop ball $B_1(v)$
and on which any deterministic one-hop-local selector is either
\begin{enumerate}
\item[(i)] \emph{infeasible}: it returns a set that fails to cover $N_2(v)$; or
\item[(ii)] \emph{feasible with cardinality ratio} $|A(v)|/\opt(v)\ge \Delta$.
\end{enumerate}
The construction uses unit costs, so the dichotomy (i)/(ii) holds verbatim for the
weighted objective. Conversely, for the \emph{cardinality} objective the selector
$A(v)=N_1(v)$ is feasible on every instance and, on every \emph{non-trivial}
instance, has ratio at most $\Delta$. Hence the worst-case ratio of any
always-feasible one-hop-local selector is exactly $\Delta$ for the cardinality
objective. (The upper bound does not transfer to the weighted
objective: a weighted optimum of $1$ forces some candidate to have cost $1$, so
selecting all of $N_1$ can cost $1+(\Delta-1)c_{\max}$ against it, and only the
lower-bound dichotomy carries over.)
\end{proposition}

\begin{proof}
Fix $d=\Delta\ge 2$ and build two graphs $G$ and $G'$ that share an identical
one-hop ball around $v$.

\paragraph{Shared core (identical $B_1(v)$).} Take the star with centre $v$ and
leaves $x_1,\dots,x_d$: vertices $\{v,x_1,\dots,x_d\}$ and edges $\{v,x_i\}$ for
$1\le i\le d$; the $x_i$ are pairwise non-adjacent. Thus in both graphs
$N_1(v)=\{x_1,\dots,x_d\}$ and the subgraph induced on $\{v\}\cup N_1(v)$ is the
same star with the same (say empty) candidate features.

\paragraph{Gadget $G$ (private terminals).} Add fresh vertices $t_1,\dots,t_d$ and
edges $\{x_i,t_i\}$. Then $\cov(x_i)=\{t_i\}$ and each $t_i$ has the unique coverer
$x_i$, so the only MPR set is $\{x_1,\dots,x_d\}$ and $\opt_G(v)=d$. Degrees:
$\deg(v)=d$, $\deg(x_i)=2$, $\deg(t_i)=1$, so the maximum degree of $G$ is $d$.

\paragraph{Gadget $G'$ (one shared terminal).} Add a single fresh vertex $t$ and
edges $\{x_i,t\}$ for all $i$. Then $\cov(x_i)=\{t\}$ for every $i$, so any one
candidate covers $N_2(v)=\{t\}$ and $\opt_{G'}(v)=1$. Degrees: $\deg(v)=d$,
$\deg(x_i)=2$, $\deg(t)=d$, so the maximum degree of $G'$ is $d$ (using $d\ge2$).

\paragraph{Indistinguishability.} By construction $B_1(v)$ is identical in $G$ and
$G'$: same centre, same candidate set, same star edges, no candidate--candidate
edges, and no terminal is within distance $1$ of $v$. Port numbering does not
separate them either: each $x_i$ has exactly two incident edges in both $G$ and
$G'$, so the ports on $B_1(v)$ can be assigned identically in the two graphs. A
one-hop-local selector therefore receives identical input and returns one and the
same set
$M:=A(v)\subseteq\{x_1,\dots,x_d\}$ in both graphs.

\paragraph{Case analysis.} Since $M\subseteq\{x_1,\dots,x_d\}$, either $|M|<d$ or
$|M|=d$.
\begin{itemize}
\item If $|M|<d$, some $x_j\notin M$. In $G$ the terminal $t_j$ is covered only by
$x_j$, so $t_j\notin\bigcup_{x\in M}\cov(x)$ and $M$ is infeasible on $G$: case (i).
\item If $|M|=d$, then $M=\{x_1,\dots,x_d\}$. In $G'$ we have $\opt_{G'}(v)=1$, so
$|M|/\opt_{G'}(v)=d=\Delta$: case (ii).
\end{itemize}
Every deterministic one-hop-local selector falls into exactly one case, proving
the dichotomy. As all costs are $1$ in the construction, the cardinality ratio
equals the weighted ratio there, so the dichotomy transfers to the weighted
objective.

\paragraph{Tightness.} Let $A(v)=N_1(v)$. Every strict two-hop node $t$ is, by
definition of distance $2$, adjacent to some one-hop neighbour of $v$, i.e.\ to
some candidate; hence $\bigcup_{x\in N_1}\cov(x)=N_2$ and $A$ is feasible on every
instance. Moreover $|N_1(v)|\le\deg(v)\le\Delta$ and, on a non-trivial instance,
$\opt(v)\ge1$, so $|A(v)|/\opt(v)\le\Delta$. Combined with case (ii), which forces
ratio $\Delta$ on some instance for \emph{any} always-feasible one-hop-local
selector (such a selector must return all of $\{x_1,\dots,x_d\}$ on $G$, hence on
$G'$), the worst-case ratio of always-feasible one-hop-local selection is exactly
$\Delta$ for the cardinality objective. That upper bound uses $|A(v)|\le\Delta$ and
is therefore specific to cardinality: under costs in $\{1,\dots,c_{\max}\}$ a
weighted optimum of $1$ forces some candidate to cost $1$, so selecting all of
$N_1$ can cost $1+(\Delta-1)c_{\max}$ against it. This is attained on $G'$ with
$c(x_1)=1$ and $c(x_i)=c_{\max}$ for $i\ge2$: then $\opt_w(v)=1$, while an
always-feasible one-hop-local selector must return all of $N_1$ (it does so on
$G$ carrying the same star costs, and the two remain indistinguishable, since
those costs live on the shared star edges of $B_1(v)$), paying exactly
$1+(\Delta-1)c_{\max}$.
\end{proof}

\noindent\emph{Remark.} A one-layer message-passing GNN whose decision at $v$ is
read out from $v$'s own representation computes a function of $B_1(v)$ and is thus
a one-hop-local selector; Proposition~\ref{prop:dich} applies to every such network
at any width. (Reading out at the candidates after $L$ layers gives each candidate
its own $L$-hop view, which is the regime Theorems~\ref{thm:constr}
and~\ref{thm:wconstr} construct in.) This is the
Remark of the main paper. The obstruction is optimality, not feasibility:
$A(v)=N_1(v)$ already certifies coverage from the one-hop ball.

\section{Proof of the constructive guarantee}
\label{sec:thm}

\subsection{Model of computation}

We use the standard message-passing model with port numbering, under which GNNs
coincide with the distributed \textsc{local} model \citep{sato2019approximation}.
The input is the attributed two-hop ball $B_2(v)$, exactly as in the main paper.
Each vertex $u$ carries a feature vector $h_u\in\mathbb{Z}^{c}$ with a constant
number $c=O(1)$ of integer channels of $O(\log\Delta)$ bits, including a node-type
channel $\tau(u)\in\{\textsc{center},\textsc{cand},\textsc{term}\}$ and an
uncovered-bit channel that is set to the value $b(t)$ at terminals and to $0$ at
$v$ and at candidates. \emph{Port numbering} equips every vertex with a fixed
ordering of its incident edges; in particular the centre $v$ assigns each
candidate $x$ a distinct $\port(x)\in\{1,\dots,\deg(v)\}$, the index of the edge
$\{v,x\}$ in $v$'s ordering.

All message functions below are gated on the endpoints' types, so information
flows only along the \emph{covering incidence edges}: the star edges $\{v,x\}$
($x\in N_1$) and the coverage edges $\{x,t\}$ ($x\in N_1$, $t\in\cov(x)$). Write
$H$ for this incidence subgraph. Any candidate--candidate or terminal--terminal
edges that the induced ball $B_2(v)$ may contain carry no covering information and
are ignored by the gating; thus $H\subseteq B_2(v)$ on the same vertex set
$\{v\}\cup N_1\cup N_2$, and no construction below ever references those edges.

One \emph{layer} updates every vertex simultaneously by
\[
  h_u \leftarrow \mathrm{UPD}\!\Big(h_u,\ \textstyle\bigoplus_{w\in N(u)}
     \mathrm{MSG}(h_u,h_w,p_{uw})\Big),
\]
where $\bigoplus\in\{\sum,\max\}$ is a permutation-invariant aggregator (both are
standard MPNN primitives), $p_{uw}$ is the \emph{sender} $w$'s port for the edge
$\{u,w\}$, so each message carries the sender's local port, the standard
port-numbering convention, and $\mathrm{UPD},\mathrm{MSG}$ are fixed-width ReLU
networks with rational weights.
``Width $O(1)$, precision $O(\log\Delta)$'' means $c=O(1)$ channels each holding an
integer of magnitude $\Delta^{O(1)}$. Because the input $B_2(v)$ is the entire
graph the network sees, no message ever references a vertex outside the two-hop
ball, \emph{regardless of the number of layers}: depth is rounds of computation
over the fixed two-hop instance, not receptive radius.

\paragraph{The greedy being simulated.} Let $\mathcal{G}$ be the coverage greedy
that maintains an uncovered set $U\subseteq N_2$ (initially $U=N_2$) and a chosen
set $S$ (initially $\emptyset$), and repeats: compute $g(x)=|\cov(x)\cap U|$ for
each $x\in N_1$; if $\max_x g(x)=0$ stop; otherwise add to $S$ the candidate
maximising the lexicographic key $\big(g(x),\will(x),\port(x)\big)$, where
$\will(x)\in\{0,\dots,W\}$ is the willingness feature, whose range $W$ is a fixed
constant of the protocol ($W=15$ in RFC~7181), and set
$U\leftarrow U\setminus\cov(x)$. We call $\mathcal G$ the \emph{cardinality
coverage greedy}: the covering core of RFC~7181 multipoint-relay selection
\citep{clausen2014olsrv2}, using the gain-first key $\big(\gain,\will,\port\big)$
of our reference implementation, in which willingness is a secondary criterion
below the marginal gain. Because $\mathcal G$ selects a maximum-marginal-gain
candidate in every round, it is exactly the classical set-cover greedy and
inherits the $(1+\ln\Delta_2)$ bound against the minimum-cardinality optimum; a
rule that instead prioritised willingness over gain would optimise a different,
lexicographic objective and would carry no cardinality guarantee. Two scope
remarks make the target precise.

\emph{Pre-passes.} The full RFC~7181 rule additionally forces every unique coverer,
forces every always-forward (\textsc{will\_always}) candidate, and masks out every
\textsc{will\_never} candidate before the coverage loop; each is a single $O(1)$ forcing layer prependable to the construction. The
unique-coverer pass is ratio-neutral: a unique coverer lies in every feasible
cover, so forcing it and recursing preserves the $(1+\ln\Delta_2)$ bound against
$\opt(v)$. The \textsc{will\_always} pass forces protocol-mandated relays that need
not lie in the unconstrained $\opt(v)$, so it preserves the bound only against the
\emph{willingness-constrained} optimum, the yardstick our experiments actually
use, since the reference integer program forces the same relays. We omit both so
that $\mathcal G$ and Theorem~\ref{thm:constr} are stated for the pure coverage
greedy against the unconstrained cardinality optimum.

\emph{Tie-break.} The third key is any fixed deterministic total order on $N_1$,
realised by the port numbering: $\port$ is an arbitrary injection
$N_1\to\{1,\dots,\deg(v)\}$. The approximation guarantee below is independent of
this choice; to make $\mathcal G$ coincide with the reference implementation's
$(\gain,\will,\text{smallest node id})$ rule one sets
$\port(x)=\deg(v)+1-\mathrm{rank}_{\uparrow\mathrm{id}}(x)$, so that maximal port
equals minimal id. Because $\port$ is injective on $N_1$, the key is a strict total
order and the argmax is unique in every round; $\mathcal{G}$ selects
$k^\ast=|S|\le|N_1|\le\Delta$ candidates.

\subsection{The three per-round lemmas}

We give the state carried between super-rounds: each terminal $t$ holds a bit
$b(t)=[\,t\in U\,]$; each candidate $x$ holds its (fixed) features $\will(x)$ and
$\port(x)$ and a selection flag $s(x)\in\{0,1\}$, initially $0$. A super-round
reads $\{b(t)\}$ and updates $\{s(x)\}$ and $\{b(t)\}$.

\begin{lemma}[marginal-coverage counting]
\label{lem:count}
One layer computes, at every candidate $x$, the exact marginal gain
$g(x)=\sum_{t\in\cov(x)}b(t)=|\cov(x)\cap U|$, an integer in $\{0,\dots,\Delta_2\}$
storable in $O(\log\Delta_2)$ bits.
\end{lemma}
\begin{proof}
In $H$ the neighbours of a candidate $x$ are $v$ and the terminals $\cov(x)$; a
terminal adjacent to $v$ would be a candidate, so no terminal neighbours $v$, and
by the type gating candidate--candidate edges are ignored. Take the sum aggregator
with the type-gated message $\mathrm{MSG}(h_x,h_w,\cdot)=b\text{-channel}(h_w)$ if
$\tau(w)=\textsc{term}$ and $0$ otherwise. Since the uncovered bit is stored on the
$b$-channel at terminals and is $0$ at $v$ and candidates, candidate $x$ receives
$\sum_{t\in\cov(x)}b(t)=|\cov(x)\cap U|$, an integer at most $|\cov(x)|\le\Delta_2$.
\end{proof}

\begin{lemma}[center-coordinated argmax]
\label{lem:argmax}
Assume each candidate holds $g(x)$ (Lemma~\ref{lem:count}), $\will(x)$ and
$\port(x)$. In $O(1)$ layers, using $O(1)$ channels of $O(\log\Delta)$-bit
precision, the network sets an indicator $I(x)\in\{0,1\}$ that equals $1$ for the
unique maximiser of $\big(g(x),\will(x),\port(x)\big)$ and $0$ for all other
candidates.
\end{lemma}
\begin{proof}
\emph{Injective, order-preserving key.} Put
\[
  \varphi(x)=C_1\,g(x)+C_2\,\will(x)+\port(x),
\]
with $C_2=\Delta+1$ and $C_1=C_2 W+\Delta+1$. These are functions of the global
constants $\Delta$ and $W$ only, never of the centre's own degree, so a
\emph{single} weight assignment serves every centre, as Theorem~\ref{thm:constr}
requires, since it asserts one network for all $v$.
This is a linear readout of the three integer channels (one layer). We claim
$\varphi(x)>\varphi(y)$ iff $(g,\will,\port)(x)>_{\mathrm{lex}}(g,\will,\port)(y)$.
Write $\delta_g=g(x)-g(y)$, $\delta_w=\will(x)-\will(y)$,
$\delta_p=\port(x)-\port(y)$, so
$\varphi(x)-\varphi(y)=C_1\delta_g+C_2\delta_w+\delta_p$, with
$|\delta_w|\le W$ and $1\le|\delta_p|\le\deg(v)-1\le\Delta-1$ (ports are distinct).
\begin{itemize}
\item If $\delta_g\ge1$: $\varphi(x)-\varphi(y)\ge C_1-C_2W-(\Delta-1)=2>0$.
\item If $\delta_g=0,\ \delta_w\ge1$: $\varphi(x)-\varphi(y)=C_2\delta_w+\delta_p
      \ge C_2-(\Delta-1)=2>0$.
\item If $\delta_g=\delta_w=0$: $\varphi(x)-\varphi(y)=\delta_p$, which has the
      sign of $\port(x)-\port(y)$.
\end{itemize}
By symmetry the converse inequalities hold, so $\varphi$ is strictly
order-preserving; since $\port$ is injective on $N_1$, $\varphi$ is injective, and
consecutive values differ by at least $1$. Its magnitude is
$\varphi(x)\le C_1\Delta_2+C_2W+\deg(v)=O(\Delta^2)$, i.e.\ $O(\log\Delta)$ bits.
(Candidate $x$ obtains $\port(x)$ in one setup layer: under the sender-port
convention, $v$'s message to $x$ carries $v$'s port for the edge $\{v,x\}$, which
is $\port(x)$ by definition; $x$ copies it into a channel once, as ports are
static.)

\emph{Center maximum.} The neighbours of $v$ in $H$ are exactly the candidates
(terminals are at distance $2$). With the $\max$ aggregator, one layer gives $v$
the value $\varphi_{\max}=\max_{x\in N_1}\varphi(x)$.

\emph{Recover the argmax.} $v$ broadcasts $\varphi_{\max}$ to its neighbours (one
layer: send $\varphi_{\max}$ along every incident edge). Each candidate $x$ then
computes
\[
  I(x)=\mathrm{ReLU}\big(\varphi(x)-\varphi_{\max}+1\big).
\]
Because $\varphi(x)\le\varphi_{\max}$ and the $\varphi$ are integers at least $1$
apart, $I(x)=1$ exactly when $\varphi(x)=\varphi_{\max}$ and $I(x)=0$ otherwise.
Injectivity makes the maximiser unique, so exactly one candidate has $I(x)=1$.
Every step is a linear map, a $\max$ aggregation, or a single ReLU unit, using a
constant number of $O(\log\Delta)$-bit channels.
\end{proof}

\begin{lemma}[gated coverage update]
\label{lem:update}
Assume each candidate holds $g(x)$ and a $\{0,1\}$-valued indicator $I$ with
exactly one candidate satisfying $I(x)=1$. Then in $O(1)$ layers the network
(a) adds the $I$-marked candidate to $S$ \emph{if and only if} its gain is
positive, and no other candidate, and (b) removes the newly covered terminals
from $U$. If $\max_x g(x)=0$ the round is
a no-op.
\end{lemma}
\begin{proof}
Let $I(x)$ be the given indicator (supplied by Lemma~\ref{lem:argmax} in the
cardinality construction and by Lemma~\ref{lem:wargmin} in the weighted one). Define the clipped
positivity bit $q(x)=\min(1,g(x))=\mathrm{ReLU}(g(x))-\mathrm{ReLU}(g(x)-1)$, which
is $1$ if $g(x)\ge1$ and $0$ if $g(x)=0$ (two ReLU units on the integer channel
$g$). The winner's selection mark is
\[
  m(x)=\mathrm{ReLU}\big(I(x)+q(x)-1\big)=I(x)\wedge q(x),
\]
the logical AND of two bits. Set $s(x)\leftarrow\max(s(x),m(x))$. Let $w$ be the
unique candidate with $I(w)=1$; every other candidate has $I=0$ and hence $m=0$.
If $g(w)\ge1$ then $q(w)=1$, so $m(w)=1$ and $w$ is added; if $g(w)=0$ then
$m(w)=0$ and nothing is added. In particular, when $\max_x g(x)=0$ we have
$q(\cdot)\equiv0$, so $m(\cdot)\equiv0$ and the round is a no-op. The argument
uses only the hypothesis that $I$ is $\{0,1\}$-valued with a unique $1$; it does
not assume $w$ maximises the gain.

For the update, each terminal $t$ aggregates $m(\cdot)$ over its candidate
neighbours: $r(t)=\max_{x:\,t\in\cov(x)}m(x)$ (one $\max$-aggregation layer). Then
$b(t)\leftarrow b(t)\cdot(1-r(t))=\mathrm{ReLU}(b(t)-r(t))$, which clears exactly
the terminals covered by the marked winner. Thus $U$ becomes
$U\setminus\cov(w)$ when a winner fires, and is unchanged otherwise.
\end{proof}

\subsection{Composition, correctness, and the approximation bound}

\begin{theorem}[constructive guarantee]
\label{thm:constr}
For every $\Delta$ there is a message-passing GNN with port numbering, depth
$O(\Delta)$, at most $\Delta$ super-rounds of $O(1)$ layers each, and $O(1)$
channels of $O(\log\Delta)$-bit precision, that at every centre $v$ outputs
$S=\{x:s(x)=1\}$ equal to the selection $\mathcal{G}(v)$ of the greedy of
\S\ref{sec:thm}. Consequently the network is coverage-correct and attains a
$(1+\ln\Delta_2)$-approximation to the minimum-cardinality MPR set, and all its
message passing stays within the two-hop ball $B_2(v)$.
\end{theorem}

\begin{proof}
One super-round is the composition of Lemma~\ref{lem:count} (one layer),
Lemma~\ref{lem:argmax} ($O(1)$ layers), and Lemma~\ref{lem:update} ($O(1)$
layers), so a super-round is $O(1)$ layers. A one-time setup pass (initialise
$b(t)=1$, $s(x)=0$, deliver $\port(x)$) is also $O(1)$ layers. Stack $\Delta$
super-rounds; total depth $O(\Delta)$, width $O(1)$, precision $O(\log\Delta)$.

\emph{Correctness by induction on rounds.} Invariant $P_r$: after $r$ super-rounds
the network state satisfies $S=S_r$ and $U=N_2\setminus\bigcup_{x\in S_r}\cov(x)$,
where $S_r$ is the set chosen by $\mathcal G$ after $r$ iterations, and
$b(t)=[t\in U]$. Base $r=0$: $S_0=\emptyset$, $U=N_2$, $b\equiv1$, established by
the setup pass. Step: assume $P_r$. By Lemma~\ref{lem:count} each candidate holds
$g(x)=|\cov(x)\cap U|$, the same gain $\mathcal G$ uses in iteration $r+1$. If
$\max_x g(x)=0$ then $\mathcal G$ has terminated ($S_{r+1}=S_r$, $U=U$) and
Lemma~\ref{lem:update} makes the round a no-op, so the network state also stays
$S_r,U$. Otherwise Lemma~\ref{lem:argmax} marks the unique lexicographic maximiser
$w$, which is exactly the candidate $\mathcal G$ selects in iteration $r+1$
(identical key, identical uniqueness). Its gain $g(w)=\max_x g(x)\ge1$ is
positive, since gain is the leading field of the key and we are in the branch
$\max_x g(x)\neq0$; Lemma~\ref{lem:update}'s positivity condition therefore
holds, and it sets
$S\leftarrow S_r\cup\{w\}=S_{r+1}$ and $U\leftarrow U\setminus\cov(w)$, restoring
$b(t)=[t\in U]$. Hence $P_{r+1}$ holds.

Since $\mathcal G$ terminates after $k^\ast\le\Delta$ iterations with
$U=\emptyset$, invariant $P_\Delta$ gives $S=S_{k^\ast}=\mathcal G(v)$: rounds
$k^\ast+1,\dots,\Delta$ satisfy $\max_x g(x)=0$ and are no-ops. Thus the network
reproduces $\mathcal G(v)$ exactly, and $\mathcal G$ is coverage-correct
($U=\emptyset$ at termination), so $S$ covers $N_2$.

\emph{Approximation.} $\mathcal G$ is the greedy for minimum-cardinality set cover
with ground set $N_2$ and sets $\{\cov(x)\}_{x\in N_1}$ of size at most
$\Delta_2$. The classical analysis of greedy set cover gives approximation ratio
$\mathcal{H}(\Delta_2)\le 1+\ln\Delta_2$, where $\mathcal{H}$ is the harmonic
number \citep{chvatal1979greedy}. The tie-break affects only which optimal-ratio run is
realised, not the bound. Since the network output equals $\mathcal G(v)$, it
inherits the $(1+\ln\Delta_2)$-approximation.

\emph{Locality.} Every message in every layer travels along an incidence edge of
$H\subseteq B_2(v)$ between $v$, a candidate, and a terminal, all of which lie in
the two-hop ball. No vertex outside $B_2(v)$ is ever referenced, at any depth; the
$O(\Delta)$ layers are rounds over the fixed two-hop instance. This is the
decoupling of depth from radius.
\end{proof}

\subsection{The weighted objective}
\label{sec:weighted}

Theorem~\ref{thm:constr} concerns the cardinality objective. The paper's
experiments optimise the \emph{weighted} objective, in which each candidate $x$
carries an integer link cost $c(x)\in\{1,\dots,c_{\max}\}$ for a fixed bound
$c_{\max}$ ($c_{\max}=10$ in our instances) and a cover $S\subseteq N_1$ is charged
$c(S)=\sum_{x\in S}c(x)$; write $\opt_w(v)$ for the minimum weight of a cover. Here
$c(x)$ is the metric of the star edge $\{v,x\}$; since each candidate has exactly
one such edge, it is carried without loss of generality as an
$O(\log c_{\max})$-bit node channel of $x$ in the attributed ball, which is
precisely how the implementation stores it.

The weighted cost-effectiveness greedy $\mathcal G_w$ is defined like $\mathcal G$
but, each round, if $\max_x g(x)>0$, adds the candidate minimising the
lexicographic key
\[
  \Big(\tfrac{c(x)}{g(x)},\ -g(x),\ -\will(x),\ -\port(x)\Big)
\]
over $\{x: g(x)\ge1\}$: smallest cost-effectiveness ratio, breaking ties toward
larger gain, then larger willingness, then the fixed deterministic order given by
the port numbering, which, setting
$\port(x)=\deg(v)+1-\mathrm{rank}_{\uparrow\mathrm{id}}(x)$ as in
\S\ref{sec:thm}, is the reference implementation's smallest-node-id rule. A
selected candidate has $g=0$ in every later round and so leaves
$\{x:g(x)\ge1\}$; the selections are therefore distinct and
$k^\ast\le|N_1|\le\Delta$, exactly as for $\mathcal G$. This is the covering
loop of the weighted cost-effectiveness baseline our experiments use, the
metric-aware analogue of the RFC~7181 rule (the RFC itself specifies no weighted
variant); exactly as for $\mathcal G$
(\S\ref{sec:thm}), we state $\mathcal G_w$ without the \textsc{will\_always}
forcing pre-pass and the \textsc{will\_never} mask that the reference
implementation additionally applies, and Theorem~\ref{thm:wconstr} is stated for
$\mathcal G_w$ against the unconstrained $\opt_w(v)$, whereas the reference integer
program both forces the always-relays and drops the \textsc{will\_never} candidates.
On the instances we evaluate, every candidate carries the default willingness, so
both differences are vacuous and the two optima coincide. By Chv\'atal's analysis of the weighted
greedy, $\mathcal G_w$ attains an $\mathcal H(\Delta_2)\le 1+\ln\Delta_2$
approximation to $\opt_w(v)$, the parameter being the largest set size
\citep{chvatal1979greedy}. The last three key fields order only candidates already
tied on $c(x)/g(x)$, so every run of $\mathcal G_w$ is a minimum-ratio greedy run and
the bound is tie-break independent, exactly as in the cardinality case.

The obstruction is that the primary key $c(x)/g(x)$ is a \emph{rational}, not the
bounded integer Lemma~\ref{lem:argmax} consumes. It is, however, a ratio of bounded
integers, so it ranges over a fixed finite set. Let
$R=\{\,c/g : c\in\{1,\dots,c_{\max}\},\ g\in\{1,\dots,\Delta\}\,\}\subset\mathbb Q$,
a set of size $|R|\le c_{\max}\Delta$ that depends only on $c_{\max}$ and the degree
bound $\Delta$, not on the instance, and let $\mathrm{rk}:R\to\{1,\dots,|R|\}$ be the
position of a value in the ascending order of $R$ (a fixed, precomputable table).
Note that equal ratios share a rank, which is the intended behaviour: $\mathcal G_w$
compares the rationals, so candidates with $c/g$ equal (say $1/2$ and $2/4$) are tied
on its primary key and must fall through to the gain tie-break.

\begin{lemma}[weighted center-coordinated argmin]
\label{lem:wargmin}
Assume costs are integers in $\{1,\dots,c_{\max}\}$ with $c_{\max}=\Delta^{O(1)}$,
and that each candidate holds $g(x)$ (Lemma~\ref{lem:count}), $c(x)$, $\will(x)$ and
$\port(x)$. In $O(1)$ layers, using $O(c_{\max}\Delta)$ channels of
$O(\log\Delta)$-bit precision, the network sets an indicator that is $1$ for the
candidate $\mathcal G_w$ selects and $0$ elsewhere, whenever $\max_x g(x)\ge1$.
Moreover $I$ is $\{0,1\}$-valued with exactly one candidate satisfying $I(x)=1$
\emph{unconditionally}, including in rounds where every gain is zero.
\end{lemma}
\begin{proof}
\emph{Ratio-rank encoding.} Lemma~\ref{lem:count} delivers $g(x)\in\{0,\dots,\Delta_2\}$
with $\Delta_2\le\Delta$, so the pair $(c(x),g(x))$ lies in the grid
$\{1,\dots,c_{\max}\}\times\{0,1,\dots,\Delta\}$. Define
\[
  \rho(x)=\begin{cases}
    |R|-\mathrm{rk}\big(c(x)/g(x)\big), & g(x)\ge1,\\[2pt]
    0, & g(x)=0,
  \end{cases}
\]
so $\rho(x)\in\{0,\dots,|R|-1\}$, the second case being exactly what the network
below emits, since no grid bump for a positive gain fires. On $g\ge1$, $\mathrm{rk}$ is strictly increasing in the ratio, so
$\rho$ is strictly \emph{decreasing} in $c(x)/g(x)$: a smaller ratio gives a larger
$\rho$. The map $(c,g)\mapsto\rho$ is a fixed function on a grid of
$c_{\max}(\Delta+1)$ cells and is realised exactly as follows. Linearise the grid by
the affine map $k=c+c_{\max}g$, absorbed into the hidden pre-activation; it is a
bijection onto $\{1,\dots,c_{\max}(\Delta+1)\}$ because $c\ge1$, so the cells of
gain $g$ occupy the block $\{c_{\max}g+1,\dots,c_{\max}(g+1)\}$ and no two cells
collide,
and put on the integer $k$ the bump
$\beta_{k_0}(k)=\mathrm{ReLU}(k-k_0+1)-2\,\mathrm{ReLU}(k-k_0)+\mathrm{ReLU}(k-k_0-1)$,
which is $1$ at $k=k_0$ and $0$ at every other integer. Summing
$\beta_{k_0}$ over the grid with coefficient $|R|-\mathrm{rk}(c/g)$ (coefficient $0$
on the $g=0$ column) gives $\rho$ exactly, using $3c_{\max}(\Delta+1)=O(c_{\max}\Delta)$
hidden units and one linear readout. On a $g=0$ input every positive-gain bump
evaluates to $0$ and the $g=0$ bump carries coefficient $0$, so the output is exactly
the integer $0$, as required. The output lies in $\{0,\dots,|R|-1\}$, i.e.\
$O(\log\Delta)$ bits for $c_{\max}=\Delta^{O(1)}$.

\emph{Reduction to Lemma~\ref{lem:argmax}.} All four fields
$\rho(x),g(x),\will(x),\port(x)$ are bounded integers. Take
$\psi(x)=C_0\rho(x)+C_1 g(x)+C_2\will(x)+\port(x)$ with $C_2=\Delta+1$ and
$C_1=C_2W+\Delta+1$ as in Lemma~\ref{lem:argmax}, and
$C_0=C_1\Delta+C_2W+\Delta+1$, again functions of $\Delta$ and $W$ alone, so one
weight set serves every centre. The verification that $\psi$ is strictly
order-preserving for the lexicographic order on $(\rho,g,\will,\port)$ and injective
(via $\port$) is the case analysis of Lemma~\ref{lem:argmax} with one extra leading
field: if $\rho(x)>\rho(y)$ then
$\psi(x)-\psi(y)\ge C_0-\big(C_1\Delta+C_2W+\Delta-1\big)=2>0$ by the choice of
$C_0$, and if $\rho(x)=\rho(y)$ the difference is exactly the three-field expression
already verified there. In particular $\psi$ is integer-valued and injective, so the
indicator $I(x)=\mathrm{ReLU}(\psi(x)-\psi_{\max}+1)$ of Lemma~\ref{lem:argmax} is
$\{0,1\}$-valued and equals $1$ for exactly one candidate \emph{unconditionally},
including in rounds where every gain is zero, a fact Lemma~\ref{lem:update} needs
to make such a round a no-op.

\emph{Zero-gain candidates are dominated.} The network's max-aggregation ranges over
\emph{all} candidates, including those with $g(x)=0$, which $\mathcal G_w$ excludes.
They are excluded automatically: such a candidate has $\rho=0$ and $g=0$, hence
$\psi(x)=C_2\will(x)+\port(x)\le C_2W+\deg(v)\le C_2W+\Delta<C_1$, whereas any
candidate with $g(x)\ge1$ has $\psi\ge C_1 g\ge C_1$. So whenever $\max_x g(x)\ge1$ the
$\psi$-maximiser has positive gain, and among those it realises $\mathcal G_w$'s key:
maximal $\rho$ (minimal ratio), then maximal $g$, then maximal $\will$, then maximal
$\port$, which equals smallest id because
$\port(x)=\deg(v)+1-\mathrm{rank}_{\uparrow\mathrm{id}}(x)$. The center
max-aggregation and broadcast-compare of Lemma~\ref{lem:argmax} then mark it in
$O(1)$ layers. Magnitude $\psi=O(|R|\Delta^2)=O(c_{\max}\Delta^3)=\Delta^{O(1)}$,
i.e.\ $O(\log\Delta)$ bits.
\end{proof}

\begin{theorem}[weighted constructive guarantee]
\label{thm:wconstr}
For every $\Delta$ and every fixed integer cost bound $c_{\max}=\Delta^{O(1)}$
there is a message-passing GNN with port numbering, depth $O(\Delta)$,
$O(c_{\max}\Delta)$ channels of $O(\log\Delta)$-bit precision, that at every centre
reproduces $\mathcal G_w$ and therefore attains a $(1+\ln\Delta_2)$-approximation to
the minimum-\emph{weight} MPR set, while all message passing stays within the
two-hop ball $B_2(v)$.
\end{theorem}
\begin{proof}
Replace, in each super-round of Theorem~\ref{thm:constr}, the integer argmax of
Lemma~\ref{lem:argmax} by the weighted argmin of Lemma~\ref{lem:wargmin}. The
counting layer (Lemma~\ref{lem:count}), the gated update (Lemma~\ref{lem:update})
with its positivity bit and idempotent write $s(x)\leftarrow\max(s(x),m(x))$, the
setup pass, and the termination and locality arguments are unchanged, with
$\mathcal G$ replaced by $\mathcal G_w$.

The induction needs one addition, because the two selection rules range over
different domains: $\mathcal G_w$ minimises over $\{x:g(x)\ge1\}$ while the
network's max-aggregation ranges over all of $N_1$. The invariant $P_r$ is
objective-independent and carries over verbatim. In the step, assume $P_r$. If
$\max_x g(x)=0$ then $\mathcal G_w$ has terminated; the unconditional clause of
Lemma~\ref{lem:wargmin} still gives a $\{0,1\}$-valued indicator, and by
Lemma~\ref{lem:update} the positivity bit is $0$ for every candidate, so
$m=I\wedge q\equiv0$ and the round is a no-op with the state unchanged. If $\max_x g(x)\ge1$, the zero-gain domination of
Lemma~\ref{lem:wargmin} puts the $\psi$-maximiser in $\{x:g(x)\ge1\}$ and makes it
$\mathcal G_w$'s selection there; Lemma~\ref{lem:update} then adds it and updates
$U$, restoring $P_{r+1}$. Since $\mathcal G_w$ terminates after $k^\ast\le\Delta$
selections, $P_\Delta$ gives $S=\mathcal G_w(v)$ with the remaining rounds no-ops.

By Chv\'atal's weighted set-cover analysis $\mathcal G_w$ is
$(1+\ln\Delta_2)$-approximate for $\opt_w$ \citep{chvatal1979greedy}, and the
network inherits the bound. Depth is $O(\Delta)$ and width is $O(c_{\max}\Delta)$
(the ratio-rank layer dominates; all other channels are $O(1)$); precision is
$O(\log\Delta)$; and, as in Theorem~\ref{thm:constr}, no message leaves $B_2(v)$.
\end{proof}

\noindent The weighted theorem costs one thing relative to the cardinality one:
width $O(1)\to O(c_{\max}\Delta)$, spent entirely on the fixed ratio-rank table, and
$O(\Delta)$ only because $c_{\max}$ is a constant of the metric encoding. Depth,
locality, and the $(1+\ln\Delta_2)$ form are unchanged, and precision remains
$O(\log\Delta)$, though the key magnitude grows from $\Theta(\Delta^2)$ to
$\Theta(c_{\max}\Delta^3)$. The bound is now against the objective the paper
actually optimises.

\noindent\emph{Consistency with the deployed model.} Theorems~\ref{thm:constr}
and~\ref{thm:wconstr} concern the fixed-rule greedies $\mathcal G$ and
$\mathcal G_w$. The selector actually trained in the paper replaces the adaptive
per-round key with a single static learned score (\S\,``A Feasible-by-Construction
Neural Selector''); it is therefore \emph{not} the network of either theorem, and
its only guaranteed property is feasibility. What the theorems establish is that a
bounded-depth GNN attains the greedy ceiling at the horizon, on the cardinality
objective (Theorem~\ref{thm:constr}) and on the weighted objective the experiments
target (Theorem~\ref{thm:wconstr}); the empirical sections then report that the
learned score's \emph{measured} weighted cost ratio ($1.030$) is lower than the
measured ratio of the deployed weighted greedy ($1.138$). On our instances every
candidate carries the default willingness, so that rule's \textsc{will\_always}
forcing and \textsc{will\_never} mask are vacuous and the measured baseline is
exactly $\mathcal G_w$. That is an empirical comparison
against the certified rule, not a certified bound for the learned selector: the
$(1+\ln\Delta_2)$ guarantee attaches to $\mathcal G_w$, and nothing in this
supplement transfers it to a learned score.

\section{Reproducibility details}
\label{sec:repro}

The main paper reports the protocol; this section records the environment, the
full hyperparameter set, and the paired significance test, so that a reader can
reproduce the numbers rather than infer them.

\paragraph{Computing infrastructure.} Training, evaluation and all synthetic-graph
experiments were run on a single workstation: AMD Ryzen 9 5900X (12 cores, 24
threads), 64\,GB RAM, Windows~11 (x86-64), \textbf{CPU only}; no GPU is used at
any point, and \texttt{torch.cuda.is\_available()} is \texttt{False} throughout.
Software: Python 3.14.3, PyTorch 2.11.0 (CPU build), PyTorch~Geometric 2.8.0,
OR-Tools 9.15.6755 (CP-SAT), NetworkX 3.6.1, NumPy 2.4.4, SciPy 1.17.1,
Matplotlib 3.10.8. Two artifacts were produced on a separate Linux workstation
because they need a C toolchain and a larger scratch disk: the OONF cross-check,
which compiles the reference RFC~7181 selection code, and the ANGLOVA
evaluation. Both are CPU-only and deterministic given the committed seeds, and
both ship with their logs and input manifests. The whole pipeline is
CPU-feasible: one training seed takes roughly four minutes at this size.

\paragraph{Hyperparameters.} The encoder is a 3-layer GATv2 with hidden width
$96$, $4$ attention heads averaged rather than concatenated, no dropout, edge
dimension $3$ and node dimension $9$, each layer followed by LayerNorm on a
residual branch; the read-out head is $\mathrm{Linear}(96,96)$, ReLU,
$\mathrm{Linear}(96,1)$, for $239{,}329$ parameters in total. Training uses Adam
at learning rate $10^{-3}$, weight decay $10^{-5}$, batch size $256$, and a
\texttt{BCEWithLogits} objective whose positive weight is computed from the
training-label balance rather than fixed. We select the reported checkpoint by
\emph{decoded validation cost ratio}, not by label accuracy, because multiple
relay sets can be optimal and label agreement is therefore a noisy proxy for the
quantity we care about; validation during training is subsampled to $800$
instances for speed. The headline and horizon arms train for $12$ epochs across
five seeds; the released checkpoint used for E3, E4 and the timings trains for
$20$. These are the values in the committed configuration files; we did not run
a hyperparameter search, so no search range is reported.

\paragraph{Randomness.} Every stochastic component takes an explicit integer
seed. Graph generation seeds the geometric sampler and the metric assignment per
graph (\texttt{scripts/gen\_data.py}), so the train, validation and test splits
are drawn from disjoint seed ranges and regenerate byte-identically. Training
seeds $0$ through $4$ set PyTorch's global generator and the data-loader
shuffling. The reported multi-seed figures vary only the training seed; the data
is held fixed.

\paragraph{Paired significance test.} Because both selectors are evaluated on the
same instances, the appropriate test is paired and does not assume normal
residuals. On the $3200$ test instances the learned selector is cheaper than the
weighted greedy on $2459$ of the $2978$ non-tied pairs and dearer on $519$, with
a mean per-instance saving of $3.32$ cost units. A one-sided Wilcoxon
signed-rank test rejects equality overwhelmingly, $W=4.85\times10^{5}$,
$p<10^{-290}$. We report this for completeness; the effect is large enough that
its direction was never in doubt, and the seed-to-seed spread reported in the
main paper ($\pm0.001$ on cost/optimum) is the more informative uncertainty.

\bibliography{references}